\documentclass{article}

\usepackage{microtype}
\usepackage{graphicx}
\usepackage{subfigure}
\usepackage{booktabs} 

\usepackage{hyperref}

\usepackage[archival]{icml2023}

\usepackage{amsmath}
\usepackage{amssymb}
\usepackage{mathtools}
\usepackage{amsthm}

\usepackage[capitalize,noabbrev]{cleveref}

\theoremstyle{plain}
\newtheorem{theorem}{Theorem}[section]

\newtheorem{lemma}[theorem]{Lemma}

\theoremstyle{definition}
\newtheorem{definition}[theorem]{Definition}

\theoremstyle{remark}

\icmltitlerunning{Sumformer: Universal Approximation for Efficient Transformers}

\usepackage{comment}
\usepackage[utf8]{inputenc} 
\usepackage[T1]{fontenc}    
\usepackage{hyperref}       
\usepackage{url}            
\usepackage{booktabs}       
\usepackage{nicefrac}       
\usepackage{microtype}      
\usepackage{xcolor}         
\usepackage{mathabx} 
\usepackage{bm} 

\newcommand{\R}{\mathbb R}
\newcommand{\N}{\mathbb N}
\newcommand{\Att}{\mathrm{Att}}
\newcommand{\AttHead}{\mathrm{AttHead}}
\newcommand{\Block}{\mathrm{Block}}
\newcommand{\T}{\mathcal{T}}

\newcommand{\norm}[1]{\left\lVert#1\right\rVert}

\newcommand{\e}{\varepsilon}

\newcommand{\w}{\omega}

\newcommand{\X}{\mathcal{X}}
\newcommand{\Y}{\mathcal{Y}}
\newcommand{\Sc}{\mathcal S}
\newcommand{\LinformerHead}{\mathrm{LinAttHead}}

\newcommand{\Ftarget}{\mathcal{F}_{\mathrm{equi}}^n(\X,\Y
)}

\newcommand{\FC}{\mathrm{FC}}

\newcommand{\PerformerHead}{\mathrm{PerAttHead}}

\begin{document}

\twocolumn[
\icmltitle{Sumformer: Universal Approximation for Efficient Transformers}




\begin{icmlauthorlist}
\icmlauthor{Silas Alberti}{LMU,Stanford}
\icmlauthor{Niclas Dern}{TUM}
\icmlauthor{Laura Thesing}{LMU}
\icmlauthor{Gitta Kutyniok}{LMU}
\end{icmlauthorlist}

\icmlaffiliation{TUM}{TUM School of Computation, Information and Technology, Technical University of Munich, Munich, Germany}
\icmlaffiliation{LMU}{Department of Mathematics, Ludwig-Maximilians-Universität München, Germany}
\icmlaffiliation{Stanford}{Department of Electrical Engineering, Stanford University, United States}

\icmlcorrespondingauthor{Silas Alberti}{salberti@stanford.edu}

\icmlkeywords{}

\vskip 0.3in
]



\printAffiliationsAndNotice{}  

\begin{abstract}
    Natural language processing (NLP) made an impressive jump with the introduction of Transformers. ChatGPT is one of the most famous examples, changing the perception of the possibilities of AI even outside the research community. However, besides the impressive performance, the quadratic time and space complexity of Transformers with respect to sequence length pose significant limitations for handling long sequences. While efficient Transformer architectures like Linformer and Performer with linear complexity have emerged as promising solutions, their theoretical understanding remains limited. In this paper, we introduce Sumformer, a novel and simple architecture capable of universally approximating equivariant sequence-to-sequence functions. We use Sumformer to give the first universal approximation results for Linformer and Performer. Moreover, we derive a new proof for Transformers, showing that just one attention layer is sufficient for universal approximation.
\end{abstract}

\section{Introduction}

The introduction of the Transformer architecture in 2017 \cite{vaswani_attention_2017} commenced a new revolution in the field of deep learning. It not only revolutionized Natural Language Processing with famous models like BERT \cite{devlin_bert_2019} and GPT-3 \cite{brown_language_2020} but also other areas like computer vision \cite{dosovitskiy_image_2021} and biology \cite{jumper_highly_2021}.

However, Transformers can become computationally expensive at scale. In many cases, the primary performance bottleneck is the attention mechanism that needs to compute a $n \times n$-Matrix, where $n \in \N$ is the length of the input sequence. Therefore, the computational complexity of a forward pass grows $O(n^2)$ with the sequence length. This establishes the sequence length as one of the major bottlenecks when using of Transformers for long sequences, which are encountered in many fields, such as NLP for processing longer documents like books, time series \cite{wen_transformers_2023}, genomics \cite{eraslan_deep_2019}, and reinforcement learning \cite{chen_decision_2021}.

To address this problem, many new architectures have been proposed \cite{child_generating_2019,wang_linformer_2020,choromanski_rethinking_2021,katharopoulos_transformers_2020,kitaev_reformer_2020,zaheer_big_2021,tay_efficient_2020,beltagy_longformer_2020}. These can be roughly divided into \emph{sparse Transformers} and \emph{efficient Transformers} \cite{tay_efficient_2020}. In some cases, the complexity can be reduced to as low as $\mathcal{O}(n)$. While, in practice, these new architectures do not match the performance of Transformers, the relative performance to the decrease in computational cost makes them promising.

Besides their empirical performance, little is known about the theoretical properties of these new architectures. Particularly, they have not yet been studied from the perspective of expressivity. This paper shows that the efficient Transformers, Linformer and Performer, are universal approximators of equivariant continuous sequence-to-sequence functions on compact sets.

\subsection{Summary of contributions}

In this paper, we introduce the \emph{Sumformer} architecture. This architecture serves as a simple tool that we can use to investigate the expressive power of Transformers and two selected efficient Transformer architectures: Linformer \cite{wang_linformer_2020} and Performer \cite{choromanski_rethinking_2021}. We chose the latter two architectures since they performed best in the Long Range Arena benchmark \cite{tay2020long}.

First, we show that the Sumformer architecture is able to approximate all continuous equivariant sequence-to-sequence functions on compact sets (Sec. \ref{section_sumformer}). We give two different proofs: A continuous proof based on the algebra of multisymmetric polynomials, and a discrete proof based on a piecewise constant approximation.

Using this result, we give a new proof of the universal approximation theorem for Transformers (Sec. \ref{Sec:TransformerProof}). This proof improves significantly upon the previous result from \cite{yun_are_2020}, by reducing the number of necessary attention layers. Our proof only needs one attention layer, whereas the number of attention layers in \cite{yun_are_2020} grows exponentially with the token dimension.

Based on this proof, we give the first proof that Linformer and Performer are universal approximators (Sec. \ref{Ch:UnivEffTrans}). This is the first universal approximation theorem for efficient Transformers, showing that despite using the efficient attention mechanisms we do not suffer from a loss in expressivity.

Our numerical experiments (Sec. \ref{Sec:Experiments}) using the Sumformer architecture show that the Sumformer architecture is not only theoretically useful but can indeed be used to learn functions using gradient descent. Furthermore, we find an exponential relation between the token dimension and the necessary latent dimension.

\subsection{Related work}\label{Sec:RelatedWorks}

This paper analyses the expressive power of newly evolving efficient Transformer architectures. Expressivity is a natural first question when investigating the possibilities and limitations of network architectures. Therefore, the question of which functions can be approximated (uniformly) with neural networks and their variance is of great interest. 

The publications mentioned in the following are by no means exhaustive but rather a selection: The first universal approximation result for neural networks dates back to 1989 with the universal approximation theorem in \cite{hornik_multilayer_1989}. Further investigations also for deeper networks were made in \cite{barron1994approximation,mhaskar1996neural,shaham2018provable}. These results were extended to functions with the rectified linear unit (ReLU) activation function in \cite{petersen_optimal_2018, lu_expressive_2017, yarotsky_error_2017, guhring2020error} and convolutional neural networks in \cite{yarotsky2022universal}. Feed forward neural networks with fewer non-zero coefficients and values that can be stored with fewer bits and therefore improve memory efficiency are investigated in \cite{bolcskei_optimal_2017}.   

The Transformer architecture has not been explored as much in the literature. We know from \cite{yun_are_2020} that Transformers are universal approximators in $L_p$, for $1 \leq p < \infty $ for continuous sequence-to-sequence functions. Moreover, it has been shown in \cite{yun_on_2020} that under certain assumptions on the sparsity pattern, sparse Transformers form universal approximators in the same setting. The expressivity of the self-attention mechanism has also been examined from a complexity theory perspective in \cite{likhosherstov_expressive_2021}. For efficient Transformer architectures, no such universal approximation results exist to our knowledge.

The main inspiration for this work is the Deep Sets architecture which shows a universal approximation theorem for invariant functions on sets \cite{zaheer_deep_2018, wagstaff_limitations_2019}. We expand on their theorems in the continuous case (Theorem 7 \& 9) and expand the theory from invariant functions on sets to equivariant functions on sequences. A similar model to Sumformer was proposed, and universality was proven in \cite{hutter_representing_2020}. However, the connection to (efficient) Transformers was not made. We build upon their proof and propose an alternative discontinuous version.
Concurrent work has given the continuous proof in higher dimension, but neither considers the expansion to equivariant sequence-to-sequence functions nor to Transformers \cite{chen_representation_2023}.
\section{Preliminaries}\label{Sec:Preliminaries}

This section describes the setting and states helpful theorems for our proofs and experiments. 

We first recall the definition of \emph{attention heads} and the \emph{Transformer block} from \cite{vaswani_attention_2017}. Afterwards, we describe how they can be changed to be more efficient with Linformer and Performer.

Furthermore, we define \emph{equivariant}, \emph{semi-invariant} functions, \emph{multisymmetric polynomials}, and \emph{multisymmetric power sums} \cite{briand_when_2004}. We also state important theorems about the relations between these concepts from \cite{hutter_representing_2020} and \cite{briand_when_2004}. Lastly, we recall an important theorem from \cite{zaheer_deep_2018}.

\subsection{Transformer}
The central part of the Transformer is the (self-)attention layer, which is able to connect every element in a sequence with every other element.

\begin{definition}[Attention Head \cite{vaswani_attention_2017}]\label{def_atthead}
Let $W_Q,W_K,W_V\in\R^{d\times d}$ be weight matrices and let $\rho:\R^d\rightarrow\R^d$ be the softmax function. A \emph{(self-)attention head} is a function $\AttHead:\R^{n\times d}\rightarrow\R^{n\times d}$ with
\begin{equation}
    \AttHead(X):=\underbrace{\rho\Big((XW_Q)(XW_K)^\top /\sqrt{d}\Big)}_{A}XW_V
\end{equation}
where $\rho$ is applied row-wise. We call $A\in\R^{n\times n}$ the \emph{attention matrix}.
\end{definition}

Computing the attention matrix $A$ has a computational complexity of $\mathcal{O}(n^2)$, thereby forming the highest cost in evaluating the Transformer. 

In the next step, we combine the attention heads to an attention layer by concatenating $h$ attention heads and multiplying them with another weight matrix $W_O$.

\begin{definition}[Attention Layer \cite{vaswani_attention_2017}]\label{def_att}
Let $h\in\N$, let $\AttHead_1,\dots,\AttHead_h$ be attention heads and let $W_O\in\R^{hd\times d}$. A \emph{(multi-head) (self-)attention layer} $\Att:\R^{n\times d}\rightarrow\R^{n\times d}$ is defined as
\begin{align}
\begin{split}
    \Att(X) &:=  [\AttHead_1(X),\dots,\AttHead_h(X)]W_O.
\end{split}
\end{align}
\end{definition}

For the Transformer architecture the attention layer is combined with fully-connected layers that are applied token-wise. Moreover, there are residual connections between all the layers \cite{he_deep_2015}. Those three components together yield the \emph{Transformer block}.

\begin{definition}[Transformer Block \cite{vaswani_attention_2017}]\label{def_block}
A \emph{Transformer block} $\Block:\R^{n\times d}\rightarrow\R^{n\times d}$ is an attention layer $\Att:\R^{n\times d}\rightarrow\R^{n\times d}$ followed by a fully-connected feed forward layer $\FC:\R^{d}\rightarrow\R^{d}$ with residual connections
\begin{equation}
    \Block(X):=X+\FC(X+\Att(X))
\end{equation}
where the fully-connected feed-forward layer $\FC$ is applied row-wise.
\end{definition}

Similar to the concept of feed-forward neural networks, we stack several Transformer blocks after each other by concatenation. The Transformer architecture is then defined as follows.

\begin{definition}[Transformer Network \cite{vaswani_attention_2017}]\label{def_transformer}
Let $\ell\in\N$ and $\Block_1,\dots,\Block_\ell$ be Transformer blocks. A \emph{Transformer network} $\T:\R^{n\times d}\rightarrow\R^{n\times d}$ is a composition of Transformer blocks:
\begin{equation}
  \T(X):=(\Block_\ell\circ\Block_{\ell-1}\circ\dots\circ\Block_1)(X).  
\end{equation}
\end{definition}

\subsection{Efficient Transformer}

To address the $O(n^2)$ bottleneck of computing the attention matrix $A$, various efficient Transformers were introduced. We chose to investigate Linformer and Performer since they stood out in the Long Range Arena benchmark \cite{tay2020long}. Both architectures only replace the attention mechanism and do not change the rest of the architecture.

\subsubsection{Linformer}

The Linformer architecture is motivated by the observation that the attention matrix $A$ is effectively low rank. This is supported by empirical evidence in actual language models and theoretical results in  \cite{wang_linformer_2020}.

The Linformer architecture utilizes the Johnson-Lindenstrauss Lemma by using linear projections $E,F\in\R^{k\times n}$ to project the key and value matrix $K=XW_K$ and $V=XW_V$ from $\R^{n\times d}$ to $\R^{k\times d}$. The entries of $E$ and $F$ are sampled from a normal distribution. The precise definition of a Linformer Attention Head is as follows:

\begin{definition}[Linformer Attention Head \cite{wang_linformer_2020}]\label{def_linatthead}
Let $k\in\N$ with $k<n$ and let $E,F\in\R^{k\times n}$ be linear projection matrices. Furthermore, let $W_Q,W_K,W_V\in\R^{d\times d}$, $\rho:\R^d\rightarrow\R^d$ be as in the Definition of a Transformer attention head \ref{def_atthead}. A \emph{Linformer attention head} is a function $\LinformerHead:\R^{n\times d}\rightarrow\R^{n\times d}$ with
\begin{equation}
    \LinformerHead(X):=
    \rho\Big((XW_Q)(EXW_K)^\top /\sqrt{d}\Big)
    FXW_V
\end{equation}
where $\rho$ is applied row-wise.
\end{definition}

Then, the new attention matrix $\overline{A}=\rho((XW_Q)(EXW_K)^\top /\sqrt{d})$ will be in $\R^{n\times k}$, giving a computational complexity of $\mathcal O(nk)$ instead of $\mathcal O(n^2).$ Using the Johnson-Lindenstrauss Lemma it is shown that when $k$ is chosen on the order of $\mathcal O(d/\e^2)$, the attention mechanism is approximated with $\e$ error.  Since $\mathcal O(d/\e^2)$ is independent of $n$, the complexity of Linformer Attention is $\mathcal O(n)$ as $n$ increases.

\subsubsection{Performer}
The key insight that motivates the Performer architecture is the fact that the attention mechanism could be more efficient if the attention matrix had no non-linearity:
\begin{equation}
    (QK^T)V=Q(K^TV)
\end{equation}
This reduces the computational complexity from $O(n^2d)$ to $O(nd^2)$. By interpreting the attention matrix as a kernel matrix, this non-linearity can be replaced by a dot product in a kernel space, enabling the following efficient attention algorithm:

\begin{definition}[Performer Attention Head \cite{choromanski_rethinking_2021}]\label{def:Performer}
Let $k\in\N$ with $k<n$, let $\w_1,\dots,\w_k\sim\mathcal{N}(0,I_d)$ and define $a:\R^d\rightarrow\R^k$ as
\begin{equation}
    a(x):=\frac{1}{\sqrt{k}}\exp\bigg(-\frac{\norm{x}^2}{2}\bigg)\big[\exp(\w_1^\top x),\dots,\exp(\w_k^\top x)\big].
\end{equation}
Furthermore, let $W_Q,W_K,W_V\in\R^{d\times d}$ be weight matrices. A \emph{Performer attention head} is a function $\PerformerHead:\R^{n\times d}\rightarrow\R^{n\times d}$ with
\begin{equation}
    \PerformerHead(X):=a(XW_Q)\Big(a(XW_K)^\top (XW_V)\Big)
\end{equation}
where $a$ is applied row-wise.
\end{definition}

With this definition, we avoid the computation of the full attention matrix, which reduces the computational complexity from $O(n^2d)$ to $O(nkd).$

\subsection{Equivariant and Semi-Equivariant Functions}

Let $\X$ and $\Y$ be the domain and range of a function, e.g., $\X=\Y=\R^d$ or $\X,\Y=[0,1]^d$ in the compact case. We call an element $X\in \X^n$ a \textit{sequence of $n$ elements} and denote $X=[x_1,\dots,x_n]$. Often, we refer to the elements $x_i$ of the sequence as \textit{points}. In the canonical case $\X \subseteq \R^{d}$, we can represent sequences $X\in\R^{n\times d}$ as matrices. We call functions of type $f:\X^n\rightarrow\Y$ \emph{sequence-to-point} functions.

\begin{definition}[Equivariance]
A sequence-to-point function $f:\X^n\rightarrow\Y$, with $\X,\Y \subset \R^d$ is \emph{equivariant} to the order of elements in a sequence if for each permutation $\pi:[n]\rightarrow[n]$:
\begin{equation}
    f([x_{\pi(1)},\dots,x_{\pi(n)}])=[f_{\pi(1)}(X),\dots,f_{\pi(n)}(X)].
\end{equation}
We write that $f\in\Ftarget$.
\end{definition}

Transformers represent sequence-to-sequence functions, but sometimes it is more convenient to work with sequence-to-point functions. To facilitate that, we recall the concept of a \textit{semi-invariant} function (see: \cite{hutter_representing_2020}).

\begin{definition}[Semi-invariance]\label{def_semiinvariance}
A sequence-to-point function $g:\X^n\rightarrow\Y$ is \emph{semi-invariant} if for each permutation $\pi:[n]\setminus\{1\}\rightarrow [n]\setminus\{1\}$:
\begin{equation}
    g([x_1,x_2,\dots,x_n])=g([x_1,x_{\pi(2)},\dots,x_{\pi(n)}]).
\end{equation}
\end{definition}

In this context, the following insight from [\cite{hutter_representing_2020}, Lemma 10] is important because it enables us to deal with equivariant sequence-to-sequence functions by looking at semi-invariant sequence-to-point functions instead:

\begin{lemma}[Equivalence of Equivariance and Semi-invariance \cite{hutter_representing_2020}]\label{lemma_semiinvariance}
A sequence-to-sequence function $f:\X^n\rightarrow\Y^n$ is equivariant if and only if there exists a semi-invariant sequence-to-point function $g:\X^n\rightarrow\Y$ such that
\begin{align}
\begin{split}
    f&([x_1,\dots,x_n])\\
    &=[g(x_1,\{x_2,x_3\dots\}),g(x_2,\{x_1,x_3,\dots\}),\dots].
\end{split}
\end{align}
\end{lemma}

\subsection{Multisymmetric Polynomials}

We discuss two different proofs for the universality of Sumformer. For the continuous proof, we use multisymmetric polynomials, which we introduce now. Our definitions are based on \cite{briand_when_2004}.

\begin{definition}[Multisymmetric Polynomial]\label{definition_multi_poly}
    Let $\X \subset \R^d$. A (real) multisymmetric polynomial in a sequence of length $n$ is a polynomial $p: \X^n \rightarrow \mathbb{R}$ in the variables $x^{(1)}_1, x^{(1)}_2, ..., x^{(n)}_d$ which is invariant in permutations of $x^{(1)}, ..., x^{(n)}$.
\end{definition}

\begin{definition}[Multisymmetric Power Sum]\label{definition_multi_sum}
    A multisymmetric power sum of multidegree $\alpha = (\alpha_1, ..., \alpha_d) \in \N^d \backslash \{0\}$ is a multisymmetric polynomial of the form:
    \begin{equation}
        p_\alpha: \X^n \rightarrow \mathbb{R}, [x^{(1)}, ..., x^{(n)}] \mapsto \sum_{i = 1}^n (x^{(i)})^\alpha
    \end{equation}
    where $(x^{(i)})^\alpha = (x^{(i)}_1)^{\alpha_1} \cdots (x^{(i)}_d)^{\alpha_d}$.
\end{definition}

The multisymmetric power sums are of interest because they can generate any multisymmetric polynomial. The following theorem which follows directly from [\cite{briand_when_2004}, Theorem 3 \& Corollary 5] shows this relationship:

\begin{theorem}[Multisymmetric Power Sums generate Multisymmetric Polynomials]\label{power_sum_generation}
    The real multisymmetric power sums in a sequence of length $n$ with multidegree $|\alpha| \coloneqq \alpha_1 + \dots + \alpha_d \leq n$ generate all real multisymmetric polynomials (in a sequence of length n), i.e. every multisymmetric polynomial p can be represented by 
    \begin{equation}
        p = \sigma(p_{\alpha^{(1)}}, ..., p_{\alpha^{(z)}})
    \end{equation}
    with a (real) polynomial $\sigma$ and the multisymmetric power sums $p_{\alpha^{(1)}}, ..., p_{\alpha^{(z)}}$.
\end{theorem}

\subsection{Deep sets}

As discussed in Section \ref{Sec:RelatedWorks}, the concept of a Sumformer, which we introduce in section \ref{section_sumformer}, is related to the concept of deep sets introduced in \cite{zaheer_deep_2018}. We also utilize the following theorem for the discontinuous proof:

\begin{theorem}[\cite{zaheer_deep_2018}, Theorem 2]\label{deepsets_discrete}
Let $Z = \{z_1, ..., z_M\}, z_m \in E$, $E$ countable and $\mathcal{Z}$ be the power set of $Z$. 
A function $f: \mathcal{Z} \rightarrow \mathbb{R}$ operating on $Z$ can be permutation invariant to the elements in $Z$, if and only if it can be decomposed in the form $\psi(\sum_{z \in Z} \phi(x))$, for suitable transformations $\phi$ and $\psi$.
\end{theorem}

\section{Sumformer}\label{section_sumformer}

We now introduce the new architecture \emph{Sumformer}. The name stems from the inherent dependence on the sum of a function evaluation of every token separately.

\begin{definition}[Sumformer]\label{definition_sumformer}
Let $d'\in\N$ and let there be two functions $\phi:\X\rightarrow\R^{d'},\psi:\X\times\R^{d'}\rightarrow\Y$. A \emph{Sumformer} is a sequence-to-sequence function $\mathcal{S}:\X^n\rightarrow\Y^n$ which is evaluated by first computing
\begin{equation}
    \Sigma:=\sum_{k=1}^n\phi(x_k),
\end{equation}
and then
\begin{equation}
    \mathcal S([x_1,\dots,x_n]):=[\psi(x_1,\Sigma),\dots,\psi(x_n,\Sigma)].
\end{equation}
\end{definition}

The Sumformer architecture is simple and can be approximated with Transformers, Linformers, and Performers. The simplicity of the architecture and the ability to prove the universality of multiple architectures using it suggests that Sumformers can also be approximated by other architectures and thereby give universal approximation theorems for them.

\section{Universal approximation}\label{Sec:Universal}

In this section, we give the main theorems of this paper. We first show that Sumformers are universal approximators for continuous sequence-to-sequence functions. This result can be used to give a new proof for the universality of Transformers and the first universal approximation results for Linformer and Performer.

Before continuing, we make an important assumption: For the rest of this paper, let $\X,\Y\subseteq\R^d$ and let $\X$ be a compact set. Note that $\X$ and $\Y$ do not need to have the same dimensionality in the following theorems. This only simplifies our notation.

\subsection{Sumformer}

We show two different proof ideas for the universal approximation by Sumformer.

The second relies on a local approximation with a piecewise constant function. This approximation allows us to choose the inherent dimension $d'=1$. Hence, we are able to choose a very small attention matrix. However, due to the discontinuous structure, we need exponentially many feed-forward layers in the sequence lengths $n$ and the token size $d$.

This problem can be circumvented with an approximation with continuous $\psi$ and $\phi$ using multisymmetric power sums from Definition \ref{definition_multi_sum}. In this case, four feed-forward layers and one attention or summing layer are sufficient. However, the inherent dimension $d'$ scales with $n^d$ - for a fixed d - in this case. Therefore, the related attention matrices also scale with $n^d$.

We investigate this trade-off further in Section \ref{Sec:Experiments} with numerical experiments.

\begin{theorem}[Universal Approximation by Sumformer]\label{theorem_sumformer}
For each function $f\in\Ftarget$ and for each $\e>0$ there exists a Sumformer $\mathcal{S}$ such that
\begin{equation}
    \sup_{X\in\X^n}\|f(X)-\Sc(X)\|_\infty<\e.
\end{equation}
\end{theorem}

\begin{proof}[Proof sketch continuous]

    We aim to use Theorem \ref{power_sum_generation}. Therefore, for every $i \in [d]$, we approximate coordinate $i$ of $f$ with an equivariant vector of polynomials $p_{i}: \X^n \mapsto \R^n$ with an accuracy of $\epsilon / d$ (as done in \cite{hutter_representing_2020}). This is possible using a version of the Stone-Weierstrass theorem from \cite{hutter_representing_2020}. Because $p_{i}$ is equivariant we can use Theorem \ref{lemma_semiinvariance} to represent $p_{i}$ by a semi-invariant polynomial $q_i:\X^n \mapsto \R$, such that $p_{i}([x_1, \ldots, x_n])=[q_i(x_1, \left\{x_2, \ldots, x_n \right\}), \ldots, q_i(x_n, \left\{x_1, \ldots, x_{n-1}\right\})$. 

   Now, we use Theorem \ref{power_sum_generation} and a representation similar to \cite{hutter_representing_2020} to represent $q_i$ using multisymmetric monomials and polynomials of multisymmetric power sums. For this, we define a function mapping to the power sums: Let $\phi:\R^d \mapsto \R^{d'}$ be the map to all $d'$ multisymmetric monomials with order $0 < |\alpha|\leq n$. The sum in the Sumformer is then represented as $\Sigma = \sum_{i=1}^n \phi (x^{(i)})$. We represent $q_i$ by
   \begin{equation}
    \psi_i(x^{(j)}, \Sigma) = \sum_{\alpha \in P} (x^{(j)})^\alpha \cdot \sigma_\alpha(\Sigma - \phi(x^{(j)}))
\end{equation}
with $P \subseteq \mathbb{N}_0^d, |P| < \infty$ and $\sigma_\alpha$ are polynomials. Finally, by setting $\psi = [\psi_1, ..., \psi_d]$, we obtain a Sumformer $\mathcal{S}$ with $\mathcal{S}(x) = [p_1(x), ..., p_d(x)]$ which therefore also fulfills the required goodness of fit.
\end{proof}

\begin{proof}[Proof sketch discontinuous]
    Instead of approximating the equivariant function $f$, we approximate the semi-invariant and uniformly continuous (since $\X$ is compact) function $g$, which represents every component as described in Theorem \ref{lemma_semiinvariance}. To be able to use Theorem \ref{deepsets_discrete} with a countable input, we approximate $g$ with a locally constant function $\overline{g}$. The used grid is of size $(1/\delta)^{nd}$ for some $\delta>0$, which depends on $\e$. The new function $\overline g$ is also semi-invariant. 

    Now, we can assign every grid point $p \in G$ a coordinate $\chi(p) = (a, b) \in [\Delta]^d \times [\Delta]^{(n - 1) \times d}$ where $\Delta = \frac{1}{\delta}$. Furthermore, we can find a function $\lambda: [\Delta]^{(n - 1) \times d} \rightarrow \mathbb{N}$ with a finite range which yields the same output if and only if the input sequences are permutations of each other.
    
    In the next step, we can use Theorem \ref{deepsets_discrete} to find $\phi^*$ and $\psi^*$ so that
     $   \lambda(b) = \psi^*\left( \sum_{i = 1}^{n-1} \phi^*(b_i) \right)$.

    Let $q: \X \rightarrow [\Delta]^d$ be the function mapping tokens to the corresponding cube-coordinate. Then by defining
    \begin{equation}
    \Sigma=\sum_{i=1}^n\phi^*(q(x))
    \end{equation}
    and
    \begin{align}
        \begin{split}
&\psi(x_1,\Sigma) \\
&:=\overline{g}\bigg(\chi^{-1}\bigg(q(x_1),\;\lambda^{-1}\Big(\psi^*\big(\Sigma-\phi(x_1)\big)\Big)\bigg)\bigg)
        \end{split}
    \end{align}
    
    we yield a Sumformer with the required goodness of fit. Note that even though $\lambda^{-1}$, in general, might not be invertible, we can find an inverse of a restriction of $\lambda$ to a subset of the domain such that properties necessary for our proof are given.
\end{proof}

\subsection{Transformer}\label{Sec:TransformerProof}

With the approximation result for Sumformers, we can now present two new proofs for the universality of Transformers. Before we give these proofs, we want to highlight the first universality theorem for Transformers from  \cite{yun_are_2020} and discuss the similiarities and differences.

\begin{theorem}[Universal Approximation by Transformer \cite{yun_are_2020}]\label{theorem_og_transformer}
Let $\e>0$ and let $1\leq p<\infty$. Then, for any continuous, permutation-equivariant function $f:\R^{n\times d}\rightarrow\R^{n\times d}$ with compact support, there exists a Transformer Network $\T$ such that
\begin{equation}
    \left(\int \|\T(X)-f(X)\|_p^p dX \right)^{1/p} \leq\e.
\end{equation}
\end{theorem}

The first noticeable difference is the fact that \cite{yun_are_2020} uses the $L_p$ norm to measure the accuracy. In our setting, we aim to understand the worst-case behavior and therefore use the supremum norm. Furthermore, \cite{yun_are_2020} also gives proofs for functions that are not equivariant by using positional encoding. Because the positional encoding is added only to the input and does not change any further points about the architecture, this can probably be applied also in our case.

Beyond the difference in the theorem setup, we also have a very different proof strategy. The proof in \cite{yun_are_2020} relies on the concept of contextual mappings. To implement these mappings, the Transformer needs $\e^{-d}$ many attention layers, where $d$ is the token size and $\e$ is the desired approximation accuracy. With our proof, we improve upon this result by showing that we only need one attention layer, which is used to represent the sum in the Sumformer.

With this information, we can now state our theorem for the universal approximation by Transformers.

\begin{theorem}[Universal Approximation by Transformer]\label{theorem_transformer}
For each function $f\in\Ftarget$ and for each $\e>0$ there exists a Transformer $\T$ such that
\begin{equation}
\sup_{X\in\X^n}\|f(X)-\T(X)\|_\infty<\e.    
\end{equation}
\end{theorem}

\begin{proof}[Proof Sketch]
   First, note that the weights in the attention matrix can be set to zero; this way, we can get feed-forward networks only. In the continuous case, $\phi$ is also continuous and can therefore be approximated with a 2-layer network by \cite{hornik_multilayer_1989}. For the discontinuous proof, we know from \cite{yun_are_2020} that we need $\mathcal O (n(1/\e)^{nd}/n!)$ many layers for the approximation.

In the following steps, we approximate the sum with an attention head. This step is equal for the continuous and discontinuous settings. However, in the discontinuous case, we can set $d'=1$. This step is also the only step we need to investigate for the Linformer and Performer proof. 

We first use a feed-forward neural network to have as input the matrix: 
 \begin{equation}\label{Eq:TransformerInput}
       \begin{bmatrix}
           1 & x_1 & \phi(x_1) & \bm{0}_{d'} \\
           & \ldots && \\
           1 & x_n & \phi(x_n) & \bm{0}_{d'}
       \end{bmatrix} \in \R^{n \times 1 + d+ 2d'}
   \end{equation} 
Then, we choose 
\begin{equation}\label{Eq:WQWK}
    W_Q=W_K=[e_1,\bm{0}_{(1+d+2d')\times (1+d+2d')}]
\end{equation}
with $e_1=[1,\bm{0}_{d+2d'}]^\top\in\R^{1+d+2d'}$ such that
    $   A= \frac 1 n 1_{n \times n}$
   and $W_V$ such that we get together with the skip connection:
   \begin{equation}\label{Eq:TransformerOutput}
       \begin{bmatrix}
           1 & x_1 & \phi(x_1) & \Sigma \\
           & \ldots && \\
           1 & x_n & \phi(x_n) & \Sigma
       \end{bmatrix} \in \R^{n \times 1 + d+ 2d'}
   \end{equation} 
   We can then, in the continuous case, apply another two layers for the approximation of the continuous $\psi$, or we need another $\mathcal O(n (1/\e)^{nd}/n!)$ many feed-forward layers to approximate the $\psi$ build in the discontinuous case.
\end{proof}

\subsubsection{Network size}

Using Sumformer, we were able to give two different constructions for the Transformer as universal approximators. We note that the construction of the attention head remains the same except for the possible choice of $d'$. When we approximate $\phi$ and $\psi$ with smooth functions, we need a larger latent dimension $d'$. In the discontinuous construction, we need more layers to approximate $\phi$ and $\psi$ but can approximate the function of interest $f$ using only $d'=1$. The same situation can be observed for the efficient Transformers as we only replace the attention heads but keep the functions $\phi$ and $\psi$ from the proof of the Transformer. There might be another way of representing functions with Sumformers. However, the current proofs suggest a trade-off between the size of the latent dimension $d'$ and the number of necessary layers. In Section \ref{Sec:Experiments}, we test the dependence of the validation loss on the relationship of $d'$ to the sequence length $n$ and the token size $d$.

\subsection{Efficient Transformers are Universal Approximators}\label{Ch:UnivEffTrans}


Using the concept of Sumformer, we can show that Linformer and Performer are universal approximators for continuous functions on a compact support. We are able to utilize the proof for Transformers as the architecture is only changed in the attention head, which forms the main computational cost of Transformer. As the rest of the architecture stays the same, this part of the proof does not need to be adapted. We start with Linformer as introduced in Definition \ref{def_linatthead}.

\begin{theorem}[Universal Approximation by Linformer]\label{theorem_linformer}
For each function $f\in\Ftarget$ and for each $\e>0$ there exist $k \in \mathcal O(d/\e^2)$ and there exist matrices $E,F\in\R^{k\times n}$ and a Linformer $\T_{\mathrm{Lin}}$ such that
\begin{equation}
    \sup_{X\in\X^n}\|f(X)-\T_{\mathrm{Lin}}(X)\|_\infty<\e.
\end{equation}
\end{theorem}
\begin{proof}
By Definition \ref{def_linatthead}, Linformer $\T_{\mathrm{Lin}}$ have the same architecture as Transformer $\T$ except for the attention head. Therefore, we can use the same construction for $\psi$ and $\phi$ as in the proof of Theorem \ref{theorem_og_transformer}. It remains to show that we can represent the sum in the Sumformer with the linear attention head as well. We now discuss how the weight and projection matrices are chosen for the approximation. Let $E=\frac{1}{n}\bm{1}_{k\times n}$ and $F=\frac{1}{k}\bm{1}_{k\times n}$, $W_Q, W_K, W_V$ as in Equation \eqref{Eq:WQWK} and we get that the Linformer attention layers maps to
\begin{align}
    \begin{split}
        &\rho((XW_Q)(EXW_K)^T)\cdot(FXW_V) =[\bm{0}_{n \times 1+d+d'}, \Sigma] 
    \end{split}
\end{align}
After applying the skip connection, we get the same output as in Equation \eqref{Eq:TransformerOutput} in Theorem \ref{theorem_og_transformer}. Therefore, we can apply the same representation for $\psi$ and get the desired approximation.
\end{proof}

Now, even though the structure and idea of Performer differ a lot from Linformer, we can use a similar strategy to show the universal approximation.

\begin{theorem}[Universal Approximation by Performer]\label{theorem_performer}
Let $k\in\N$ with $k<n$. For each function $f\in\Ftarget$ and for each $\e>0$ there exists a Performer $\T_{\mathrm{Per}}$ such that
\begin{equation}
    \sup_{X\in\X^n}\|f(X)-\T_{\mathrm{Per}}(X)\|_\infty<\e.
\end{equation}
\end{theorem}

\begin{proof}
As in the proof for the Linformer attention layer we use the fact that the Performer $\T_{\mathrm{Per}}$ only differs from a Transformer $\T$ by the choice of the attention head. Therefore, we now build a Performer attention head which is able to approximate the sum for the Sumformer.  

We choose the same $W_Q$ and $W_K$ as in Equation \eqref{Eq:WQWK}. Next, we fix the vectors $w_1,\dots,w_k$ in $a$ in the Performer Definition \ref{def:Performer}. Then, because all rows are the same and $a$ is applied row-wise, $a(XW_Q)a(XW_K)^\top = \lambda\cdot\bm{1}_{n\times n}$ for some $\lambda\in\R.$ 

In contrast, to the previous proof, we need to add another feed-forward layer after the attention layer. We choose the weight matrix to be $W=\frac{1}{\lambda n}I_{(1+d+2d')}$ and the bias $b=\bm{0}_{1+d+2d'}$. Then, we get an output of 
\begin{align}
    \begin{split}
        & Wa(XW_Q)a(XW_K)^\top(XW_V)+ b \\ & =[\bm{0}_{n \times 1+d+d'},\Sigma]^T.
    \end{split}
\end{align}
With the skip connection we get the desired input for $\psi$ and are able to use the same approximation for $\psi$ as in Theorem \ref{theorem_transformer}.
\end{proof}

\section{Numerical Experiments}\label{Sec:Experiments}

We implemented two different Sumformer architectures and tested them on approximating analytically given (i.e., non-real-world) functions. Both architectures consist of three components: one representing $\phi$, one representing $\psi$, and the last combining the two as described in Definition \ref{definition_sumformer}.

The function $\psi$ is represented by a Multi-layer perceptron (MLP) in both architectures. The representation of $\phi$ differs: The first model uses the $\phi$ we constructed in the proof of Theorem \ref{theorem_sumformer} (Polynomial Sumformer), whereas the second one uses an MLP again (MLP Sumformer).

Each MLP we used consisted of five hidden layers of 50 nodes. We use the ReLU activation function.

\begin{figure*}[t!]
\vskip 0.1in

\begin{center}
\begin{subfigure}          
\centering\includegraphics[width=0.98\columnwidth]{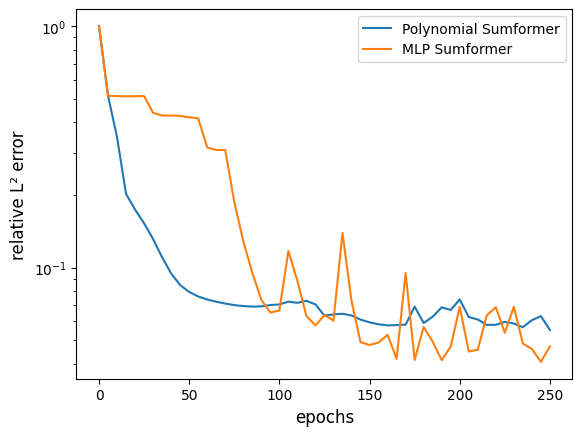}
        \end{subfigure}
        \hfill
        \begin{subfigure}
        \centering\includegraphics[width=0.98\columnwidth]{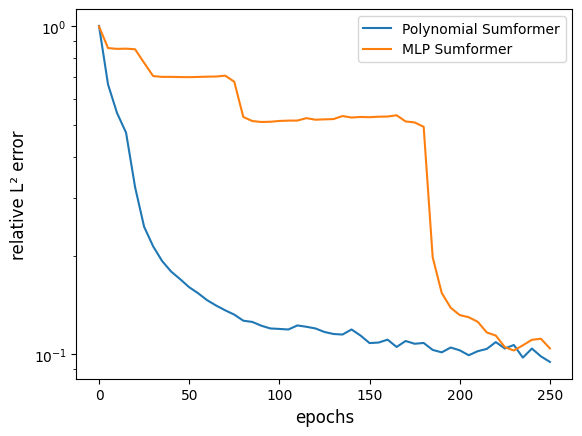}
        \end{subfigure}
        \caption[ Examples of validation errors (every 5 epochs; relative $L^2$ error) of the two Sumformer models over a training run on 2000 data points with $n = 5, d = 4$. The equivariant functions were defined by semi-invariant functions of different kinds: \emph{(left)} polynomial-type, \emph{(right)} non-polynomial-type. ]
        {\fontsize{9}{11} Examples of validation errors (every 5 epochs; relative $L^2$ error) of the two Sumformer models over a training run on 2000 data points with $n = 5, d = 4$. The equivariant functions were defined by semi-invariant functions of different kinds: \emph{(left)} polynomial-type, \emph{(right)} non-polynomial-type. } 
        \label{fig_function_approximation}
        \end{center}
        \vskip -0.1in
\end{figure*}

We trained our two models (using the same latent dimension $d'$) on approximating multiple equivariant functions (assuming $\X = [0, 1]^d$): two polynomial-type and two non-polynomial-type functions. The results (Fig.\ref{fig_function_approximation}) show that the previous results are not just theoretical: Sumformer architectures can approximate a variety of functions.

It is interesting to note that the two Sumformers perform approximately equally well on most functions we approximated (polynomial \& non-polynomial type). Based on this, we observe that the construction used in the continuous proof of Theorem \ref{theorem_sumformer} is indeed able to learn our benchmark functions using gradient descent.

Furthermore, we observe that the validation loss of the Polynomial Sumformer is smoother and decreases in a more stable way than that of the MLP Sumformer. In contrast, the validation loss of the MLP Sumformer often jumps to drastically lower levels over just a few epochs and is relatively flat apart from that. This phenomenon could be explained by the interaction of the two disjoint trainable components (MLPs).

\begin{figure}[ht!]
\vskip 0.1in
\begin{center}
\centerline{\includegraphics[width=0.98\columnwidth]{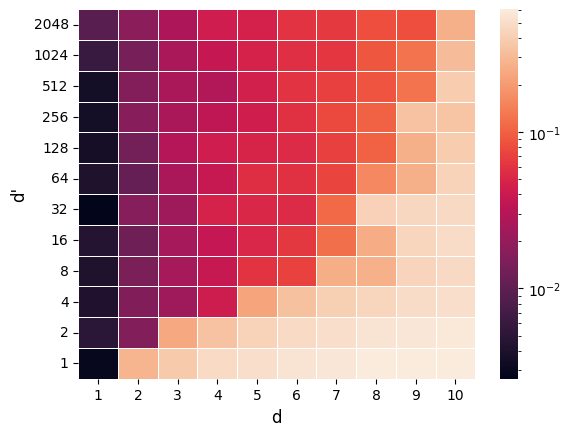}}
\caption{\fontsize{9}{11} Development of best validation errors (relative $L^2$ error) over multiple training runs (2000 data points; 200 epochs; $n = 3$) with exponentially increasing latent dimension $d'$ for ten dimensions $d$. The equivariant functions (for every $d$) were defined by the following polynomial-type semi-invariant function
        $g: [x^{(1)}, ..., x^{(3)}] \mapsto x^{(1)} + 7 \cdot (x^{(1)})^2 + 3 \cdot x^{(1)} \cdot (x^{(2)} + x^{(3)})^3$
        where products and sums are to be understood component-wise.}
\label{second_experiment}
\end{center}
\vskip -0.1in
\end{figure}

We also tested how changing the dimension $d'$ (see Definition \ref{definition_sumformer}) in the MLP Sumformer impacts the best validation loss over a fixed number of epochs while holding $n, d$ and the function to approximate constant. The results (Fig. \ref{second_experiment}) show - as expected - that higher dimensions $d'$ generally lead to better approximation. Furthermore, when changing $d$ linearly, we have to make non-linear - presumably exponential - changes to the size of $d'$ to achieve significantly diminishing returns on further increasing $d'$.

This finding is particularly interesting as the continuous proof of Theorem \ref{theorem_sumformer} needs $d' = \begin{pmatrix}
    n + d\\
    d
\end{pmatrix} - 1 = \frac{(n + d)!}{d! n!} - 1 = \prod_{i = 1}^d \frac{n + i}{i} - 1$ in $\phi$ for a fixed $n$. This suggests that the empirical performance aligns with the theory.

\section{Conclusion}

We have seen that the efficient Transformers, Linformer, and Performer, are able to represent all equivariant continuous sequence-to-sequence functions on compact sets arbitrarily well. Due to the simplicity of the Sumformer architecture on which the proofs are based, it seems likely that further research can use similar techniques to show that other Transformer architectures and state space models are also universal approximators.

In addition, we offered a new proof for universal approximation by Transformer and were able to reduce the necessary number of non-zero attention layers to only one.

In our experiments, we showed that the construction from our continuous proof of universal approximation by Sumformer is tractable and indeed able to approximate given functions using gradient descent. Furthermore, our numerical results about the impact of the latent dimension $d'$ of a Sumformer in relation to the token size $d$ nicely relate to the required size of the latent dimension in our continuous proof. 

Lastly, we note that a significant limitation of our continuous proof is that (for a fixed token size $d$) the size of the attention matrix scales with $n^d$. In other words: Although for a \textit{fixed} model dimension $d'$ the computational cost scales linearly in $n$, for achieving universal approximation the required dimension $d'$ grows polynomially in $n$ and correspondingly the overall computational cost. In the discontinuous setting, we were able to keep the latent dimension small but had to scale the number of feed-forward layers accordingly. It would be interesting to improve on this result and analyze the trade-off further in future research.

\section*{Acknowledgements}

LT and GK acknowledge support from the German Research Foundation in the frame of the priority programme SPP 2298. SA appreciates the support by the Stanford Graduate Fellowship. GK is also grateful for partial support by the Konrad Zuse School of Excellence in Reliable AI (DAAD), the Munich Center for Machine Learning (BMBF) as well as the German Research Foundation under Grants KU 1446/31-1 and KU 1446/32-1 and under Grant DFG-SFB/TR 109, Project C09 and the Federal Ministry of Education and Research under Grant MaGriDo.

\clearpage

\bibliographystyle{icml2023}
\bibliography{main}

\newpage
\appendix
\onecolumn

\section{Proofs of the universal approximation results for Sumformer}

In this section we give the details of the continuous and discontinous proofs of Theorem \ref{theorem_sumformer}.

\begin{proof}[Discontinuous case]
By Lemma \ref{lemma_semiinvariance}, there exists a semi-invariant function $g:\X^n\rightarrow\Y$ such that $f(X)=[g(x_1,\{x_2,\dots,x_n\}), \ldots, g(x_n,\{x_1, \ldots, x_{n-1}\})]$. Since $f$ is continuous, the component functions $f_1,\dots,f_n$ are also continuous and thus also $g$. The compactness of $\X$ implies that $\X^n$ is compact and therefore $g$ is uniformly continuous. Without loss of generality, let the compact support of $g$ be contained in $[0,1]^{n\times d}$. Then, we define a piece-wise constant function $\overline{g}$ by
\begin{equation}
    \overline{g}(X)=\sum_{\bm{p}\in\mathcal{G}}g(\bm{p})\bm{1}\{X\in C_{\bm{p}}\},  
\end{equation}

where the grid $\mathcal{G}:=\{0,\delta,\dots,1-\delta\}^{n\times d}$ for some $\delta:=\frac{1}{\Delta}$ with $\Delta\in\N$ consists of cubes $C_{\bm{p}}=\prod_{i=1}^n\prod_{k=1}^d[\bm{p}_{i,k},\bm{p}_{i,k}+\delta)$ with corresponding values $g(\bm{p})\in\Y$ for each $\bm{p}\in\mathcal{G}$. Because $g$ is uniformly continuous, there exists for each $\e>0$ a $\delta>0$ such that
\begin{equation}
    \sup_{X\in\X^n}\|g(X)-\overline{g}(X)\|_\infty<\e.
\end{equation}
We next show that that $\overline{g}$ is semi-invariant.  Since $g$ is semi-invariant, we have $g([x_1,x_{\pi(2)},\dots,x_{\pi(n)}])=g([x_1,x_2,\dots,x_n])$ for any permutation $\pi:[n]\setminus\{1\}\rightarrow [n]\setminus\{1\}$. With $\bm{p}=[p_1,\dots,p_n]$, we can write $\pi(\bm{p})=[\bm{p}_1,\bm{p}_{\pi(2)},\dots,\bm{p}_{\pi(n)}]$ and get $g(\bm{p})=g(\pi(\bm{p}))$. Moreover, we get $X\in C_{\bm{p}}\Leftrightarrow \pi(X)\in C_{\pi(\bm{p})}.$ Hence, for any $X\in C_{\bm{p}}$, we get
\begin{equation}
    \overline{g}(X)=g(\bm{p})=g(\pi(\bm{p}))=\overline{g}(\pi(X)).
\end{equation}

Now, we want to to represent $\overline{g}$ using an appropriate $\Sc$. While it is trivial to match each $X$ to its corresponding $\bm{p}$ such that $X\in C_{\bm{p}}$, it is more difficult to find the corresponding cube of $X$ when only being able to use $x_1$ and the aggregated $\Sigma$. 

To achieve this, we will use the following strategy: Recall that $\Delta\in\N$ is the number of cubes in each dimension. We can assign each grid point $\bm{p}\in\mathcal{G}$ a coordinate $\chi(\bm{p})=(a,\bm{b})\in[\Delta]^d\times[\Delta]^{(n-1)\times d}.$ The map $\chi:\mathcal{G}\rightarrow [\Delta]^d\times[\Delta]^{(n-1)\times d}$ is bijective and the first part of the coordinate $a\in[\Delta]^d$ can be constructed from $x_1$ by quantizing it in each dimension. Let $q:\X\rightarrow[\Delta]^d$ be this quantization function such that $q(x_1)=a$.

Let us now find a way to choose $\phi$ and $\psi$ such that we can reconstruct $\bm{b}$ from $\Sigma$. We can treat $\bm{b}$ as a sequence of length $n-1$ and write $\bm{b}=[b_1,\dots,b_{n-1}]$ with $b_i\in[\Delta]^d$. Since there are finitely many 
$\bm{b}\in[\Delta]^{(n-1)\times d}$, we can enumerate all $\bm{b}$ using a function $\lambda:[\Delta]^{(n-1)\times d}\rightarrow\N$. Moreover, let us choose $\lambda$ to be invariant to permutations of $[b_1,\dots,b_{n-1}]$, i.e. for all permutations $\pi:[n-1]\rightarrow [n-1]$ we have $\lambda([b_1,\dots,b_{n-1}])=\lambda([b_{\pi(1)},\dots,b_{\pi(n-1)}]),$  but we let $\lambda$ always assign different values to $\bm{b}_1,\bm{b}_2$ if they are not a permutation of each other. Although this prevents $\lambda$ from being injective, all cubes with the same value under $\lambda$ have the same value under $\overline{g}$, due to semi-invariance, i.e. for a fixed $a\in[\Delta]^d$ and for all $n$ in the range of $\lambda$ the inverse is well defined and we can evaluate
\begin{equation}
\overline{g}\Big(\chi^{-1}\big(a,\lambda^{-1}(n)\big)\Big).
\end{equation}
Now, $\lambda$ is an invariant sequence-to-point function and since $[\Delta]^d$ is countable, we can utilize Theorem \ref{deepsets_discrete} (note that we use multisets of a fixed size here, to which the proof in \cite{zaheer_deep_2018} can be easily extended) to find $\phi^*:[\Delta]^{(n-1)\times d}\rightarrow\R$ and $\psi^*:\R\rightarrow\N$ such that
$$\lambda(\bm{b})=\psi^*\left(\sum_{i=1}^{n-1}\phi^*(b_i)\right)$$
With the quantization function $q$ we set $\phi(x):=\phi^*(q(x))$ and define
\begin{equation}
    \Sigma=\sum_{i=1}^n\phi^*(q(x)).
\end{equation}
We can than recover $\lambda(\bm{b})$ by
\begin{equation}
    \lambda(\bm{b})=\psi^*\left(\Sigma-\phi(x_1)\right).
\end{equation}
Now, we can define $\psi$ such that the related $\mathcal{S}$ is equal to $\overline{g}:$
\begin{equation} \psi(x_1,\Sigma):=\overline{g}\bigg(\chi^{-1}\bigg(q(x_1),\;\lambda^{-1}\Big(\psi^*\big(\Sigma-\phi(x_1)\big)\Big)\bigg)\bigg).  
\end{equation}
Since we chose $\overline{g}$ to uniformly approximate $g$ and thereby each component of $f$ up to $\e$ error, this implies that $\Sc$ uniformly approximates $f$ up to $\e$ error.
\end{proof}

\begin{proof}[Continuous case]
As before we have that the compactness of $\mathcal{X}$ implies that $\mathcal{X}^n$ is compact and without loss of generality, we can assume that the compact support of $f$ is  contained in $[0, 1]^{n \times d}$.

Now, for every $i \in [d]$, we approximate coordinate $i$ of $f$ with an equivariant vector of polynomials $p_{i}: \X^n \mapsto \R^n$ with an accuracy of $\epsilon / d$ (as done in \cite{hutter_representing_2020}). This is possible using a version of the Stone-Weierstrass theorem from \cite{hutter_representing_2020}. Because $p_{i}$ is equivariant we can use Theorem \ref{lemma_semiinvariance} to represent $p_{i}$ by a semi-invariant polynomial $q_i:\X^n \mapsto \R$, such that $p_{i}([x_1, \ldots, x_n])=[q_i(x_1, \left\{x_2, \ldots, x_n \right\}), \ldots, q_i(x_n, \left\{x_1, \ldots, x_{n-1}\right\})]$. 

   Now, we use Theorem \ref{power_sum_generation} and a representation similar to \cite{hutter_representing_2020} to represent $q_i$ using multisymmetric monomials and polynomials of multisymmetric power sums. For this, we define a function mapping to the power sums:  Let 
\begin{equation}
    \phi: [0, 1]^d \rightarrow \mathbb{R}^{d'}, x \mapsto \begin{pmatrix}
        x_1^{1} x_2^{0} \cdots x_d^{0}\\
        x_1^{2} x_2^{0} \cdots x_d^{0}\\
        \vdots\\
        x_1^{\alpha_1} x_2^{\alpha_2} \cdots x_d^{\alpha_d}\\
        \vdots\\
        x_1^{0} x_2^{0} \cdots x_d^{n}\\
    \end{pmatrix}
\end{equation}
where $\alpha = (\alpha_1, ..., \alpha_d)$ runs over all multidegrees with order $0 < |\alpha| \leq n$. The sum in the Sumformer is then represented as $\Sigma = \sum_{i=1}^n \phi (x^{(i)})$. 

   By Theorem \ref{power_sum_generation} the function
\begin{equation}
    s_j(x^{(i \neq j)}) = \sigma \left( \sum_{i \neq j} \phi(x^{(i)}) \right)
\end{equation}
with $\sigma$ being a polynomial function can fit any multisymmetric polynomial in the variables $x^{(i \neq j)}  \coloneqq \{x^{(1)}, ..., x^{(j - 1)}, x^{(j + 1)}, ..., x^{(n)} \}$ perfectly.
   
   We can therefore represent $q_i$ by
   \begin{equation}
    \psi_i(x^{(j)}, \Sigma) = \sum_{\alpha \in P} (x^{(j)})^\alpha \cdot \sigma_\alpha(\Sigma - \phi(x^{(j)}))
\end{equation}
with $P \subseteq \mathbb{N}_0^d, |P| < \infty$ and $\sigma_\alpha$ are polynomials. 

By setting $\psi = [\psi_1, ..., \psi_d]$, we obtain a Sumformer $\mathcal{S}$ with $\mathcal{S}(x) = [p_1(x), ..., p_d(x)]$ which is able to approximate $f$ sufficiently well.
\end{proof}

\section{Proofs of the universal approximation results for Transformer}

Now we give the detailed proof of the universality of Transformers from Theorem \ref{theorem_transformer}.

\begin{proof}
    We use the triangular inequality to divide the approximation in two steps. We first approximate $f$ by a Sumformer $\Sc$ and then show that the Sumformer can be approximated by a Transformer $\T$, i.e.
    \begin{equation}
        \sup_{X \in \X^n} \|f(X)-\T(X)\|_\infty \leq \sup_{X \in \X^n} \|f(X)-\Sc(X)\|_\infty + \sup_{X \in \X^n} \|\Sc(X)-\T(X)\|_\infty
    \end{equation}
    For the first summand we have from Theorem \ref{theorem_sumformer} that there is a Sumformer $\Sc$ which approximates $f$ to an accuracy of $\e/2$. The Sumformer has the inherent latent dimension $d'$.
    
    We now turn to the second summand and construct a Transformer that is able to approximate the Sumformer to $\e/2$ accuracy. Transformers are constructed as described in Definition \ref{def_transformer}.   
    Because of the structure with $X + \FC(X+\Att(X))$, we can set the attention for the first layers to zero. Thereby, we obtain feed-forward layers without attention. 

    The Transformer is then constructed as follows.
      We have the input $X =[x_1, \ldots, x_n]^\top \in \X^n$ with $x_i \in \R^{1 \times d}$ and map it with a feed-forward from the right to 
    \begin{equation}
        \begin{bmatrix}
            x_1, x_1 \\
            \cdots \\
            x_n, x_n
        \end{bmatrix} \in \R^{n \times 2d}.
    \end{equation}
    We can then find a two layer feed-forward network such that it acts as the identity on the first $n$ components and approximates the function $\phi$. The approximation with two feed forward layers of $\phi$ is possible because of the universal approximation theorem \cite{hornik_multilayer_1989}. In the discontinuous setting we need more layers to approximate $\phi$. Therefore, after three feed-forward layers we get
    \begin{equation}
            \begin{bmatrix}
            x_1, \phi(x_1) \\
            \cdots \\
            x_n, \phi(x_n)
        \end{bmatrix} \in \R^{n \times (d+d') }.
    \end{equation}
    Before, we get to the attention layer we add one more layer from the right
    $\FC:\R^{d+d'}\rightarrow\R^{1+d+2d'}$ with 
\begin{equation}
    W=\begin{bmatrix}
   \bm{0}_{d\times 1}  &  I_d & \bm{0}_{d\times d'} &  \bm{0}_{d\times d'} \\
   \bm{0}_{d'\times 1} & \bm{0}_{d'\times d} &
I_{d'}
   & \bm{0}_{d'\times d'}
\end{bmatrix}\in\R^{(d+d')\times(1+d+2d')}
\end{equation}
and $b=[\bm{1}_n,0_{n\times(d+2d')}]$. Using these transformations, we get as output after the first step:
\begin{equation}
    X_1 = \begin{bmatrix}
        1 & x_1 & \phi(x_1) & \mathbf{0}_{d'}\\
       & \cdots&& \\
        1 & x_n & \phi(x_n) & \mathbf 0_d
    \end{bmatrix} \in \R^{n \times 1+d+2d' }
\end{equation}
    Note that these steps are the same for the efficient Transformers. 
    
    Now, we turn to the attention head to represent the sum $\Sigma = \sum_{i=1}^n \phi(x_i) \in \R^{d'}$. First we choose $W_Q=W_K=[e_1,\bm{0}_{(1+d+2d')\times (1+d+2d')}]\in\R^{(1+d+2d')\times(1+d+2d')}$ for $e_1=[1,\bm{0}_{d+2d'}]^\top\in\R^{1+d+2d'}$, such that 
    \begin{equation}
        A = \rho((X_1W_Q)(X_1W_K)^\top) = \frac 1 n \mathbf 1_{n \times n}.
    \end{equation}
    The matrix $A$ will then be multiplied with $X_1W_V$. We can choose 
    \begin{equation}
        W_V=\left[\begin{array}{cc}
\bm{0}_{(1+d)\times(1+d+d')} & \bm{0}_{(1+d)\times d'}  \\
\bm{0}_{d'\times(1+d+d')} & n\cdot I_{d'} \\
\bm{0}_{d'\times(1+d+d')} & \bm{0}_{d'\times d'}  
\end{array}\right]\in\R^{(1+d+2d')\times (1+d+2d')}.
    \end{equation}
    The output of this attention layer is
\begin{equation}
    [\bm{0}_{1+d+d'}, \Sigma]^\top.
\end{equation}
Then, we apply a residual connection and obtain
\begin{equation}
    [1,x_i,\phi(x_i), \Sigma]^\top.
\end{equation}
Last, we implement $\psi$. For the discontinuous case, we first compute $q(x_i)$. Then, we map a finite set of values to another finite set of values for which we can use Lemma 7 in \cite{yun_are_2020}. Hence, we need to add another $O(n\left(\frac{1}{\e}\right)^{dn}/n!)$ feed-forward layers for the approximation of $\psi$. In the continuous case this can be avoided because of the continuity of $\psi$, we can approximate it with the universal approximation theorem \cite{hornik_multilayer_1989} with $2$ feed-forward layers.
\end{proof}

\section{Deep Sets}

Sumformers are related to the concept of deep sets introduced in \cite{zaheer_deep_2018}. For the discrete proof we use Theorem \ref{deepsets_discrete}. However, there is also a version for uncountable inputs which we highlight here:

\begin{theorem}[\cite{zaheer_deep_2018}, Theorem 9]
    Assume the elements are from a compact set in $\R^{d}$, i.e. possibly uncountable, and the set size is fixed to $M$. Then any continuous function operating on a set $X$, i.e. $f : \R^{d \times M} \rightarrow \R$ which is permutation invariant to the elements in $X$ can be approximated arbitrarily close in the form of $\psi(\sum_{x \in X} \phi(x))$, for suitable transformations $\phi$ and $\psi$.
\end{theorem}

The fundamental differences of the previous theorem to our work are that we consider \emph{equivariant}, continuous sequence-to-\emph{sequence} functions. This difference is the reason why we need a second parameter in $\phi$.

\end{document}